\documentclass[letterpaper, 10 pt, conference]{ieeeconf}  

\IEEEoverridecommandlockouts                              

\usepackage{xcolor}
\usepackage{amssymb}
\usepackage{graphicx}
\usepackage{mathtools}
\usepackage{amsmath}
\usepackage{gensymb}
\usepackage{float}
\usepackage{multirow}
\usepackage{makecell}
\usepackage{tabularx}
    \newcolumntype{L}{>{\raggedright\arraybackslash}X}
\usepackage[utf8]{inputenc}
\usepackage[english]{babel}

\usepackage{amsthm}

\usepackage{subcaption}
\let\oldemptyset\emptyset
\let\emptyset\varnothing

\newcommand{\taninv}{\tan^{-1}}
\DeclareMathOperator{\sign}{sign}
\newcommand{\no}{\noindent}

\newcommand{\mc}[1]{\mathcal{#1}}
\newcommand{\bb}[1]{\mathbb{#1}}

\newtheorem{prop}{Proposition}[section]

\newcommand\Tstrut{\rule{0pt}{2.6ex}}         
\title{\LARGE \bf
CoMet: Modeling Group Cohesion for Socially Compliant Robot Navigation in Crowded Scenes

}

\author{Adarsh Jagan Sathyamoorthy, Utsav Patel, Moumita Paul, Nithish K Sanjeev Kumar, Yash Savle, \\ and Dinesh Manocha
}

\begin{document}

\maketitle
\thispagestyle{empty}
\pagestyle{empty}

\begin{abstract}
We present CoMet, a novel approach for computing a group's cohesion and using 
that to improve a robot's navigation in  crowded scenes. Our approach uses a novel cohesion-metric that builds on prior work in  social psychology. We compute this metric by utilizing various visual features of pedestrians from an RGB-D camera on-board a robot. Specifically, we detect characteristics corresponding to proximity between people, their relative walking speeds, the group size, and interactions between group members. We use our cohesion-metric to design and improve a navigation scheme that accounts for different levels of group cohesion while a robot moves through a crowd. We evaluate the precision and recall of our cohesion-metric based on perceptual evaluations. We highlight the performance of our social navigation algorithm on a Turtlebot robot and demonstrate its benefits in terms of multiple metrics: freezing rate (57\% decrease), deviation (35.7\% decrease), and path length of the trajectory(23.2\% decrease).

\end{abstract}

\section{Introduction}
Mobile robots  are increasingly being used in crowded scenarios in indoor and outdoor environments. Applications for these robots include surveillance, delivery, logistics, etc.  In such scenarios, the robots need to navigate in an unobtrusive manner and also avoid issues related to  sudden turns or freezing~\cite{freezing1}. Moreover, the robots need to integrate well with the physical and social environments.

Extensive research in social and behavioral psychology suggests that crowds in real-world scenarios are composed of (social) groups. A group is generally regarded as a meso-level concept and corresponds to two or more pedestrians with similar goals over a short or long period of time. As a result, the pedestrians or agents in a group exhibit similar movements or behaviors. 
It is estimated that up to 70\% of observed pedestrians in real-world crowds are part of a group~\cite{dyn-group-behavior,crowd-ped-dynamics-empirical}. Therefore, it is important to understand group characteristics and dynamics to perform socially-compliant robot navigation~\cite{socially-aware,WB1,frozone}. 

The problem of efficient robot navigation among pedestrians has been an active area of research.
Most existing robot navigation algorithms consider walking humans or pedestrians as separate obstacles~\cite{JHow1,socially-aware, frozone,densecavoid}. Some techniques tend to predict trajectories of each pedestrian using learning-based methods~\cite{densecavoid} but do not account for the influence of group characteristics on individuals. This could lead to obtrusive trajectories that may cut through groups of friends or families. Other methods use simple and conservative methods to detect locally sensed clusters of pedestrians and compute paths around them~\cite{frozone}. However, they do not work well as the crowd density increases.

One characteristic of groups that could be utilized to address these problems is the social-cohesion or the collective behavior of the group's members. This is directly linked to the inter-personal relationships between group members. For example, a group of friends or family has higher cohesion than a group of strangers~\cite{friendsVsstrangers, self-loafing-group-cohesiveness}. Cohesion is inversely related to the \textit{permeability} of the group in social settings, i.e. whether another individual can cut through the group while walking \cite{group-permeability}. Many theories have been proposed in psychology and sociology to identify the human behaviors or features that are good indicators of group cohesion. Such features include proximity between group members \cite{hall-proxemics}, walking speed \cite{friendly-dyads}, group size \cite{group-permeability}, context or environment \cite{scene-ind-group-profiling}, etc. Estimating cohesion could help a robot plan a better or more socially compliant trajectory based on the context.  For example, in dense crowds (i.e. pedestrian density is more than $1$ person/$m^2$), the robot could navigate around a group that has high cohesion or the robot could move between members of a group that has low cohesion, similar to how humans navigate in crowded scenarios.



\begin{figure}[t]
      \centering
      \includegraphics[width=\columnwidth,height=6.25cm]{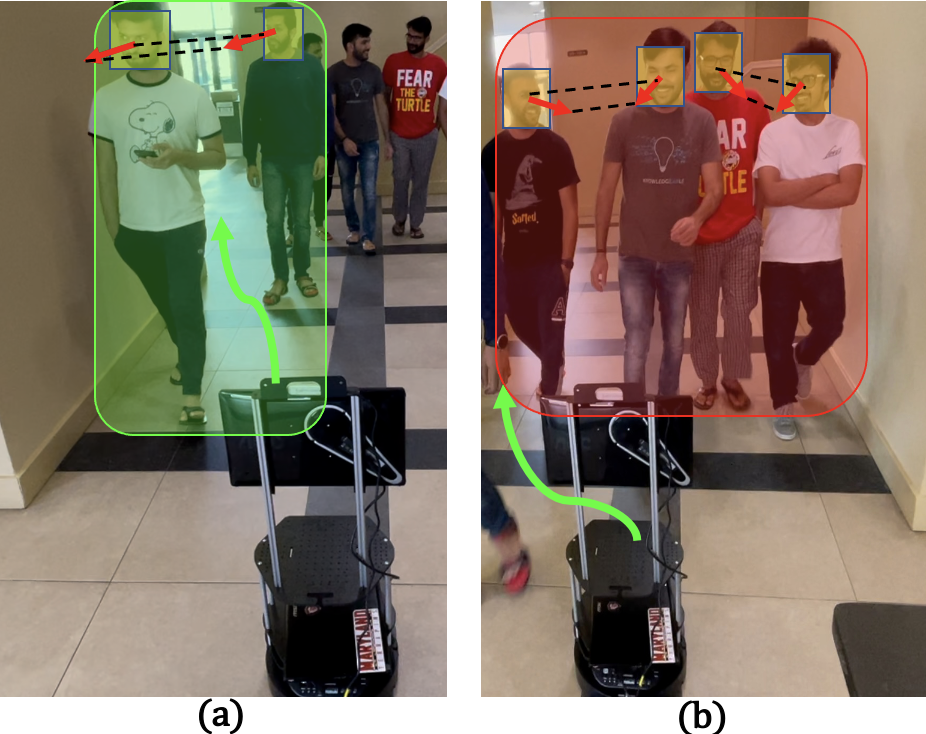}
      \caption {\small{Our novel navigation method uses our CoMet metric to compute a collision-free trajectory for a robot in real-world scenarios. CoMet identifies groups in crowds and detects intra-group proximity, walking speed, group size and interactions to estimate a group's cohesion. (a) In dense scenarios, our navigation algorithm identifies a low-cohesion group (green bounding box) and navigates between its group members (green path) by assuming human cooperation for navigation; (b) Our method detects a high cohesion group (red bounding box) and plans a trajectory around it. Overall, our method improves social-compliance and the naturalness of the trajectory (Section \ref{sec:comet-objective}).}}
      
      
      \label{fig:cover-image}
      \vspace{-15pt}
\end{figure}


\textbf{Main Contributions:} We present a novel algorithm to perform socially compliant navigation in crowded scenes. Our approach uses perception algorithms to identify groups in a crowd using visual features. We also present a novel group cohesion metric and efficient algorithms to compute this metric in arbitrary crowds using deep learning. We combine our cohesion metric with learning-based techniques to generate trajectories that tend to follow the social norms. Some of the novel components of our approach include:
\begin{itemize}

\item We present CoMet, a novel metric for estimating group cohesion. Our approach is based on social psychology studies and exploits features such as proximity between people, walking speeds, group sizes, and interactions. Our method uses an RGB-D sensor to detect groups and these visual features. CoMet has a near 100\% precision and recall when identifying low-cohesion groups, when evaluated in real-world pedestrian or crowd datasets.

\item We present a novel CoMet-based navigation method that accounts for group cohesion, while ensuring social-compliance in terms of naturalness, large deviations and freezing behaviors. Our formulation assumes human cooperation in dense crowds and plans less conservative trajectories than prior methods. We prove that the deviation angles computed by our method are less than or equal to the deviation angles computed using a prior social navigation algorithm~\cite{frozone}.

\item We implement CoMet on a real Turtlebot robot equipped with a commodity RGB-D sensor and demonstrate improvements in terms of social navigation. Our qualitative evaluations in dense scenes indicate that CoMet accurately identifies the cohesion in different groups (see Fig.\ref{fig:cover-image}). This enables the Turtlebot robot to navigate through a group based on our cohesion metric. Compared to prior social navigation algorithms, we demonstrate improved performance in terms of following metrics: freezing rate (up to 57\% decrease), path deviation or turns (up to 35.7\% decrease), and path length (up to 23.2\% decrease). 

\end{itemize}

\section{Related Works}
In this section, we briefly review prior work in robot navigation among crowds, pedestrian and group detection, and group interactions. 

\subsection{Clustering-based Group Detection}
Many techniques have been proposed in computer vision for pedestrian and group detection in a crowd. The first step in group detection is to detect individuals in the images or videos. Methods for this step include many deep learning-based approaches for pedestrian detection and tracking~\cite{YOLOv3} and improved methods for high density crowds ~\cite{densepeds}. These methods have been extended from individual pedestrian detection to group detection~\cite{disc-groups-imgs,social-groups-videos}. These group-based methods typically use different kinds of clustering based on the proximities between people, their trajectories and their velocities to segregate them into groups \cite{social-groups-videos,crowd-features-video-seq}.

\subsection{Group Behavior and Interaction Detection}
Different techniques have been proposed for detecting group behaviors and interactions in computer vision and social psychology~\cite{scene-ind-group-profiling,comparison-behavior-features} as well as event identification \cite{cluster-crowd-behavior}. Behavior detection and event identification involve the analysis of different features (e.g., collectiveness, stability, uniformity) that represent how people move and interact in a crowd. These include individuals, groups, leaders, followers, etc. They also involve detecting scenarios where groups either merge together or split while walking or running.  

Other relevant techniques detect interactions among people in a group based on F-formations \cite{disc-groups-imgs,interaction-murino}. These algorithms estimate features such as people's body and head poses and identify the individuals who are facing each other. Our approach is complimentary to these methods and extends them by using many other features, including proximity, walking, and interaction, to gauge group cohesion.  





\subsection{Robot Navigation and Social Compliance}
Many recent works have focused on socially-compliant navigation \cite{Humanaware-survey,socially-aware,WB1,frozone,Alahi,sociosense,kim2015brvo}. The underlying goal is to design methods that not only compute collision-free trajectories but also comply with social norms that increase the comfort level of  pedestrians in a crowd. At a broad level, the three major objectives of social navigation are comfort, naturalness, and high-level societal rules~\cite{Humanaware-survey}. For example, a robot needs to avoid movements that are regarded as obtrusive to pedestrians by following rules related to how to approach and pass a pedestrian~\cite{socially-aware,frozone}. Other techniques are based on modeling intra-group interactions~\cite{Alahi} or by learning from real-world static and dynamic obstacle behaviors~\cite{robot-nav-crowd-drl}.

Many techniques for social navigation have been proposed based on reinforcement learning (RL) or inverse reinforcement learning (IRL). The RL-based methods \cite{JHow1, JHow2, JiaPan1, frozone} mostly focus on treating each pedestrian as a separate obstacle to avoid collisions, sudden turns, or large deviations.
IRL methods are driven by real-world natural crowd navigation behaviors~\cite{kitani,pfeiffer} and are used to generate trajectories with high levels of naturalness. However, they can result in unsafe trajectories and may not work well as the crowd density increases. 
Some methods model pedestrian behaviors by learning about their discrete decisions and the variances in their trajectories~\cite{socially-compliant-1}.
Other works have modeled human personality traits \cite{sociosense} or pedestrian dominance \cite{ped-dominance} based on psychological characteristics for trajectory prediction and improved navigation. Bera et al.~\cite{entitivity} present an algorithm to avoid negative human reactions to robots by reducing the entitativity of robots.  Our work on modeling group cohesion is complimentary to these methods.
 


\section{Background}
In this section, we give an overview of prior work in social psychology, pedestrian tracking and robot navigation that is used in our approach. We also introduce the symbols and notation used in the paper.

\subsection{Social Psychology} \label{sec:why-choose-features}
We use four features based on prior work in social psychology to estimate the cohesion of a group. We give a brief overview of each of these features.

\textbf{Proximity:} Proximity is chosen based on the proxemics principles established by  Hall~\cite{hall-proxemics}. The underlying theory states that humans have an intimate space, a social and consultative space and a public space when interacting with others. We extend this idea to unstructured social scenarios where people or pedestrians walk, stand, or sit together. In general, humans maintain a closer proximity to other people with whom they closely interact (high cohesion). Many techniques have been proposed for simulating pedestrian movement~\cite{curtis2016menge} and collision queries~\cite{govindaraju2005quick}.


\textbf{Walking Speed:} \cite{friendly-dyads} studied individual and mixed gender groups' walking speeds in a controlled environment and observed significantly slower speeds when people walk with their romantic partners. 
\textit{Assertion 1:} A slower-than-average walking speed in a group indicates a close relationship between group members (high cohesion). 

\textbf{Group Size:} \cite{group-permeability} analyzed the perceptions of people when passing through a group of two and four people in a university hallway. It was observed that people tend to penetrate through the 4-person group less than the 2-person group.
\textit{Assertion 2:} Humans perceive the cohesion of a bigger group to be higher (implying lower permeability) than that of a smaller group. Permeability of a group is a measure of the resistance that a moving non-group entity faces while passing between the members of a group. 

\textbf{Interactions:} Jointly Focused Interactions (JFI) \cite{interaction-1961} entail a sense of mutual activity and engagement between people and imply their willingness to focus their attention on others. Therefore, it is chosen as an indicator for cohesion. We extrapolate JFI to signify a higher level of cohesion between group members. Visually, interactions can be detected by estimating peoples' 3-D face vectors \cite{interaction-murino} and detecting when the vectors point towards each other.

\subsection{Notations and Definitions}
We highlight the symbols and notation used, in Table \ref{tab:symbol_defn}. We use \textit{i}, \textit{j}, and \textit{k} to represent indices. All distances, angles and velocities are measured relative to a rigid coordinate frame attached to the camera (on the robot) used to capture the scene. The X-axis of this frame points outward from the camera and the Y-axis points to the left, with its origin at the center of the image. We use such a representation since our overall approach is local, and has no global knowledge of the environment.
 
The time interval between two consecutive RGB images in the stream is $\Delta t$.

\begin{table}[]
\centering
\begin{tabularx}{\linewidth}{|c|L|} 
\hline
\textbf{Symbols} & \textbf{Definitions}  \\
\hline
$\textbf{p}^{i, t}$ & Position vector of person i at time t relative to the robot/camera coordinate frame. $\textbf{p}^{i, t} = [x^{i, t}, y^{i, t}]$\\
\hline
$\textbf{v}^{i, t}$ & Walking velocity vector of person i at time t relative to the robot/camera coordinate frame. $\textbf{v}^{i, t} = [\Dot{x}^{i, t}, \Dot{y}^{i, t}]$ \\
\hline
$I_{rgb}^t$ & RGB image captured at time instant t of width w and height h. \\
\hline
$I_{depth}^t$ & Depth image captured at time instant t of width w and height h. \\
\hline
$\bb{B}^i$ & Bounding box of person with ID i. \\
\hline
$[x^{\bb{B}^i}_{cen}, y^{\bb{B}^i}_{cen}]$ & Centroid of the bounding box $\bb{B}^i$. \\
\hline
$\mathbf{x}^{i, t}$ & State vector used in Kalman filter for walking vector estimation for person i at time instant t. $\mathbf{x}^{i, t} = [x^{i, t} \,\, y^{i, t} \,\, \Dot{x}^{i, t} \,\, \Dot{y}^{i, t}]^T$ \\
\hline
$ID^t$ & Set of person IDs detected in $I^t_{rgb}$. \\
\hline
$n^k$ & Number of members in Group $G^{k,t}$. \\
\hline
N & Number of people detected in an RGB image. \\
\hline
$\mathbf{f}_p^i, \mathbf{f}_o^i$ & Vectors for person i's face position and orientation relative to the camera coordinate frame. \\ 
\hline
\begin{tabular}{@{}c@{}}$K_p, K_w, K_s$ \\ $K_i$ \end{tabular} & Proportionality and weighing constants for each feature \\
\hline
$CH()$ & Convex Hull function with points as its arguments. \\
\hline
\end{tabularx}
\caption{\small{List of symbols used in CoMet and their definitions.}}
\label{tab:symbol_defn}
\vspace{-10pt}
\end{table}

\subsection{Frozone} \label{sec:frozone-background}
Frozone \cite{frozone} is a navigation method that tackles the Freezing Robot Problem (FRP)~\cite{freezing1}  arising in crowds. At the same time, it can generate  trajectories that are less obtrusive to pedestrians. The underlying algorithm computes a Potential Freezing Zone (PFZ), which corresponds to a configuration of obstacles where the robot's planner halts the robot and starts oscillating for a period of time when it deems that all velocities could lead to a collision. Frozone's formulation is conservative and not suitable for dense crowd navigation ($> 1$ person/$m^2$). In addition, it treats pedestrians as individual obstacles. The main steps in Frozone's formulation are: (1) Identifying \textit{potentially freezing} pedestrians who could cause an FRP based on their walking vectors and predicting their positions after a time horizon $t_{h}$; (2) Constructing a Potential Freezing Zone ($PFZ_{froz}$) as the convex hull of the predicted positions $\hat{\textbf{p}}^{i, t + t_h}_{pf}$ (see 5-sided convex polygon in Fig. \ref{fig:proof}) for all the \textit{potentially freezing} pedestrians. $PFZ_{froz}$ is formulated as, 

\vspace{-10pt}
\begin{equation}
    PFZ = Convex Hull (\hat{\textbf{p}}^{i, t + t_h}_{pf}), \quad i \in {1,2,...,P_f.}
\end{equation}
and (3)Computing a deviation angle for the robot to avoid $PFZ_{froz}$ if its current trajectory intersects with it. $PFZ_{froz}$ corresponds to the set of locations where the robot has the maximum probability of freezing and being obstructive to the pedestrians around it. The deviation angle to avoid it is computed as 

\vspace{-10pt}
\begin{equation}
    \phi_{froz} = min(\phi_1, \phi_2),
    \label{eqn:frozone-deviation}
\end{equation}
\no where $\phi_1$ and $\phi_2$ are given by,
\begin{gather}
    \phi_1 = \underset{R_{z,\phi_1}\textbf{v}_{rob}t_{h} \notin PFZ_{froz}}{\operatorname{argmin}} \left(dist(R_{z,\phi_1}\cdot \textbf{v}_{rob}\cdot t_{h}, \,\, \textbf{g}_{rob})\right), \label{eqn:phi1}\\
    \phi_2 = \taninv(y^{near, t} / x^{near, t}), \quad \phi_2 \ne 0.
    \label{eqn:phi2}
\end{gather}

\no Here, $R_{z,\phi_1}$ is the 3-D rotation matrix about the Z-axis (perpendicular to the plane of the robot), $\textbf{v}^{rob}, \textbf{g}^{rob}$ represent the current velocity and the goal of the robot, and $[x^{near, t}, y^{near, t}]$ denotes the current location of the nearest freezing pedestrian relative to the robot. This point is in the PFZ's exterior. For navigating the robot towards its goal, and handling static obstacles and dense crowds, a Deep Reinforcement Learning (DRL)-based method \cite{JiaPan1} is used. However, the resulting navigation may cut through groups regardless of their cohesion (see Fig. 1(a)). As a result, the robot's trajectory may not be socially compliant. 
\section{CoMet: Modeling Group Cohesion}
In this section, we present our group cohesion metric, which first classifies pedestrians into groups and then measures their closeness or cohesion. Our method runs in real-time, taking a continuous stream of RGB and depth images as input and detecting the group features highlighted in Section \ref{sec:why-choose-features}. Our overall approach based on these features is shown in Fig. \ref{fig:sys-arch}.

\begin{figure}[t]
      \centering
      \includegraphics[width=7.0cm,height=6.0cm]{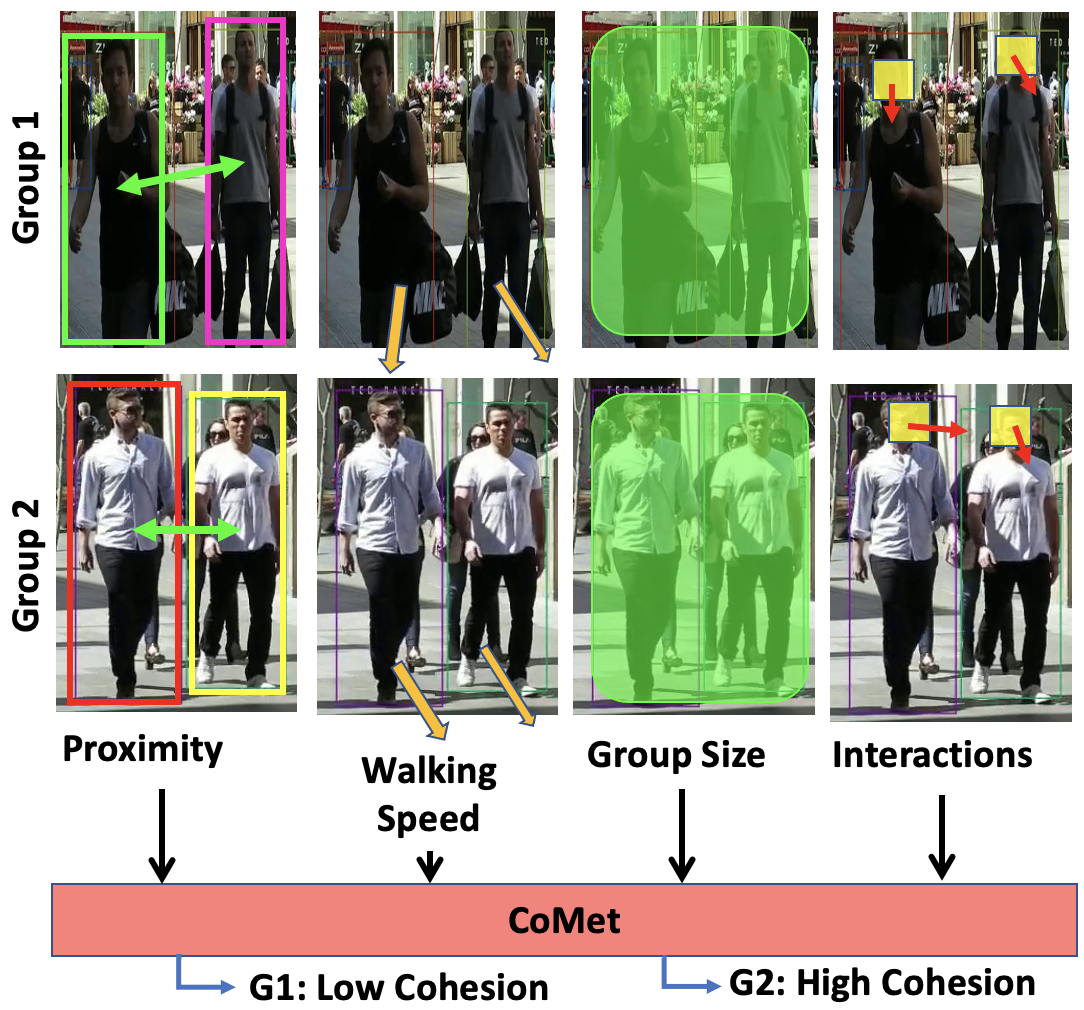}
      \caption {\small{Computation of Group Cohesion Metric: We use four features to compute the metric. We highlight how our approach is used on two different groups of pedestrians at a given time instant. These groups are shown in different rows. Group 1 consists of two pedestrians in close proximity, walking away from each other. Their distance is increasing in subsequent frames and they are looking in different directions. This implies little interaction and low group cohesion. Group 2 has two people walking together with their faces oriented towards each other, which indicates high interaction and a high group cohesion.}}
      \label{fig:sys-arch}
      \vspace{-15pt}
\end{figure}



\subsection{Detecting Group Features} \label{sec:detect-group-features}
In this section, we first describe how we track and localize people, detail conditions for a set of people to be classified as a group, and then explain efficient techniques to detect group features from RGB and depth images.

\subsection{Pedestrian Tracking and Localization} 
A key issue in detecting the features mentioned in Section \ref{sec:why-choose-features} is to first detect, track, and localize each pedestrian position relative to the camera frame in a continuous stream of RGB images. We use YOLOv5~\cite{yolo} and Deep Sort~\cite{deepsort} algorithms to detect and track people, respectively. YOLOv5 outputs a set of bounding boxes $\mc{B}= \{ \bb{B}^{i}$\} for each detected pedestrian \textit{i} in an RGB image at time instant t (denoted as $I_{rgb}^t$). $\bb{B}^{i}$ is denoted using its top-left and bottom-right corners in the image-space or pixel coordinates. In addition, we also assign a unique integer number as an ID for each detected pedestrian. 

Next, to accurately localize people, the distance of each detected person relative to the camera coordinate frame must be estimated. To this end, we use a depth image $I_{depth}^t$, every pixel of which contains the proximity (in meters) of an object at that location of the image. The pixels in $I_{depth}^t$ contain values between a minimum and maximum distance range, which depends on the camera used to capture the image.

\subsubsection{Group Classification} 
Let us consider any set of people's IDs $G^{k, t} \subseteq ID^t$. At any time t, if the following conditions hold, 

\vspace{-10pt}
\begin{gather}
        \|\mathbf{p}^{i,t} - \mathbf{p}^{j,t} \|_2 \le \Gamma \quad \forall \, i, j \in G^{k, t} \\
        \|\mathbf{p}^{i,t} - \mathbf{p}^{j,t} \|_2 \ge \|\mathbf{p}^{i,t} + \mathbf{v}^{i,t} - (\mathbf{p}^{j,t} + \mathbf{v}^{j,t}) \|_2 .
\end{gather}

\no then the set $G^{k, t}$ is classified as a \textit{group} in the image $I^t_{rgb}$. Here, $\Gamma$ is a distance threshold set manually. The first condition ensures that people are close to each other and the second condition ensures that the group members walk in the same direction. When $|\textbf{v}^i\|, \|\textbf{v}^j\| = 0$ (static groups), only the first condition is used for grouping.

\vspace{5pt}
\subsubsection{Estimating Proximity} 
To estimate the proximity between people at time instant t, first the bounding boxes detected in the RGB image by YOLOv5 are superimposed over the depth image. To estimate the distance of a person \textit{i} from the camera ($d^{i, t}$), the mean of all the pixel values within a small square centered around $[x^{\bb{B}^i, t}_{cen}, y^{\bb{B}^i, t}_{cen}]$ is computed. The angular displacement $\psi^{i, t}$ of person i relative to the camera can be computed as, $\psi^{i, t} = \left(\frac{\frac{w}{2} - x^{\bb{B}^i, t}_{cen}}{w}\right) * FOV_{RGBD}$. Here $FOV_{RGBD}$ is the field of view of the RGB-D camera. Person i's location relative to the camera can be computed as $[x^{i, t} \,\, y^{i, t}]$ = $d^{i, t}$ * [$\cos{\psi^{i, t}} \,\, \sin{\psi^{i, t}}$]. 
\vspace{5pt}
The distance between a pair of people i and j can then be computed as, 
$dist(i, j) = \sqrt{(x^{i, t} - x^{j, t})^2 + (y^{i, t} - y^{j, t})^2}$.

\vspace{10pt}

\subsubsection{Estimating Walking Speed and Direction} \label{sec:walk-vec} 
To estimate the $i^{th}$ person's walking vector $\mathbf{v}^i$ in $I^t_{rgb}$, we use a Kalman filter with a constant velocity motion model. All the detected people in $I^t_{rgb}$ (with their IDs stored in the set $ID^t$) are modeled using the state vector $\mathbf{x}^t$ defined in Table \ref{tab:symbol_defn}. If $ID^t$ contains IDs that were not present in $ID^{t - \Delta t}$, we initialize their corresponding state vectors $\mathbf{x}^t$ with constant values. For all the pedestrians who were detected in previous RGB images, i.e., with previously initialized states, we update their states using the standard Kalman prediction and update the steps \cite{kalman-filter}. We use a zero mean Gaussian noise with a pre-set variance to model the process and sensing noise. 


\subsubsection{Estimating Group Size} 
The size of the group can be trivially computed as the number of IDs in the set $G^{k, t}$.

\subsubsection{Detecting Interactions}
We use two 3-D vectors to represent the position and orientation of a person's face in $I^t_{rgb}$. We use OpenFace \cite{openface} to localize person i's face position relative to the camera coordinate frame ($\mathbf{f}_p^i$) on an RGB image and to obtain a unit vector ($\mathbf{f}_o^i$) for the face orientation. Two individuals are considered to be interacting if their face positions and orientations satisfy the following condition:

\vspace{-15pt}
\begin{equation}
    \|\mathbf{f}_p^i - \mathbf{f}_p^j \|_2 > \|\mathbf{f}_p^i + \mathbf{f}_o^i - (\mathbf{f}_p^j + \mathbf{f}_o^j) \|_2 .
    \label{eqn:interaction-condition}
\end{equation}

\no This condition checks if the distance between two people's face positions is greater than the distance between the points computed by extrapolating the face orientations (see dashed lines in Fig. \ref{fig:cover-image}), if they are facing each other. We reasonably assume that non-interacting people do not face each other.

\subsection{Cohesion Components}
We now discuss how the detected group features can be used for cohesion estimation. Using multiple features to compute cohesion is advantageous in scenarios where not all features are properly able to be sensed. 


\subsubsection{Proximity Cohesion Score} 
We use the average distance between group members to model Hall's proxemics theory as previously extrapolated. As observed in Section \ref{sec:why-choose-features}, cohesion between people is inversely proportional to the distance between them. Therefore, the cohesion score due to proximity is formulated as the reciprocal of the mean distance between group members as

\vspace{-10pt}
\begin{equation}
C_p(G^{k, t}) = K_p \cdot \frac{n^k}{\sum\limits_{\substack{i, j \in G^{k, t} \\ i \ne j}} dist(i, j)}. 
\label{eqn:prox-cohesion}
\end{equation}
\vspace{-10pt}


\subsubsection{Walking Speed Cohesion Score} 
Based on the discussion in Section \ref{sec:why-choose-features}, we next compare the average walking speeds of each group with the average walking speed of all the detected people in $I^t_{rgb}$. Therefore, the cohesion score for a walking group ($\|\textbf{v}^j\| \ne 0 \,\, \forall j \in G^{k, t}$) due to its walking speed is formulated as

\vspace{-10pt}
\begin{equation}
    C_w(G^{k, t}) = K_w \cdot \left( \frac{\sum\limits_{\forall i \in ID^t} \|\textbf{v}^i\|}{N} \right) / \left( \frac{\sum\limits_{\forall j \in G^{k, t}} \|\textbf{v}^j\|}{n^k} \right). \\
\label{eqn:walk-speed-cohesion}
\end{equation}

\no This reflects assertion 1 made in Section \ref{sec:why-choose-features}, since cohesion is inversely proportional to walking speed. The average walking speed of the entire scene is included in this formulation to normalize out the effects of crowding in the scene. If $\sum\limits_{\forall j \in G^{k, t}}\|\textbf{v}^j\| = 0$, i.e., the group is static, then $C_w(G^{k, t}) = K_w \cdot \eta$, where $\eta$ is a user-set large constant value. 

\subsubsection{Group Size Cohesion Score}
Based on assertion 2 in Section \ref{sec:why-choose-features}, the cohesion of a group \textit{k} is directly proportional to the group size ($n^k$). Therefore, the group size cohesion score is computed as

\vspace{-5pt}
\begin{equation}
    C_s(G^{k, t}) = K_s \cdot n^k.
    \label{eqn:grp-size-cohesion}
\end{equation}




\subsubsection{Interaction Cohesion Score}
The interaction condition between any two people in a group (Equation \ref{eqn:interaction-condition}) can be applied to all pairs in a group, and it's contribution to the cohesion score of a group can be re-written as

\vspace{-10pt}
\begin{equation}
    C_i(G^{k, t}) =  K_i \cdot \frac{1}{n^k} \cdot \sum\limits_{\substack{i \ne j \\ i, j \in G^{k, t} }}\frac{\sign(\theta_{ij})}{\cos{\theta_{ij}}} \quad \theta_{i,j} \in [-\frac{\pi}{4}, \frac{\pi}{4}].
    \label{eqn:interaction-cohesion}
\end{equation}
\no Here $\theta_{ij}$ is the angle between face orientation vectors $\mathbf{f}_o^i$ and $\mathbf{f}_o^j$ in the X-Y plane of the camera coordinate system. $\sign()$ is the signum function. $\theta_{ij}$ is limited to $[-\frac{\pi}{4}, \frac{\pi}{4}]$ since face orientations are accurate in this range. Intuitively, we want the cohesion score to be positive and greater than 1, when people are facing towards each other and negative otherwise. Therefore, we choose the ratio $\frac{\sign(\theta)}{\cos{\theta}}$, as it belongs to the range $[-\sqrt{2}, -1) \cup \{0\} \cup [1, \sqrt{2}]$ when $\theta_{i,j} \in [-\frac{\pi}{4}, \frac{\pi}{4}]$. Since $\cos()$ is an even function, the $\sign()$ function ensures that the formulation is sensitive to the sign of the angle.





\subsection{CoMet: Overall Group Cohesion Metric}
Using the individual cohesion scores in Equations \ref{eqn:prox-cohesion}, \ref{eqn:walk-speed-cohesion}, \ref{eqn:grp-size-cohesion},\ref{eqn:interaction-cohesion}, the total cohesion score for a group at time t ($G^{k, t}$) is given as

\vspace{-10pt}
\begin{equation}
    C_{tot}(G^{k, t}) = (C_p + C_w + C_s + C_i). 
\label{eqn:final-score}
\end{equation}
\no Here, $G^{k, t}$ is omitted in the RHS for readability. Note that $K_p, K_w, K_s, K_i$ weigh the different features before adding them to the total cohesion score. If any of the features are not detectable, their contribution to $C_{tot}$ will be zero. This acts as a measure of confidence, as our approach is able to better compute a group's cohesion when more features are detected. Our formulation is not learning-based due to the lack of extensive datasets with annotations of cohesion or related metrics for groups.

\begin{prop}
The value of the overall cohesion metric $C_{tot}(G^{k,t})$ for a group is bounded.
\end{prop}
\begin{proof}
The proof to the proposition follows from the fact that $C_p, C_w, C_s, C_i$ are bounded. The value of $C_p \in (0, K_p \cdot \Gamma]$, since $\Gamma$ is used as a threshold to group people. The maximum value of $C_w$ is $K_w\cdot\eta$, a large finite constant that is used when a group is static. $C_s$ is bounded above by $K_g \cdot n^k$, which is finite. $C_i \in [-\sqrt{2}K_i, -K_i) \cup [K_i, \sqrt{2}K_i]$. 
\end{proof}
We use these bounds on the cohesion metric to compute appropriate thresholds that are used to categorize groups as low-, medium- and high-cohesion groups.
\section{Cohesion-based Navigation} \label{sec:comet-navigation}
In this section, we present our socially-compliant navigation algorithm, which uses the group cohesion metric.

\begin{figure}[t]
      \centering
      \includegraphics[width=7.5cm,height=6.5cm]{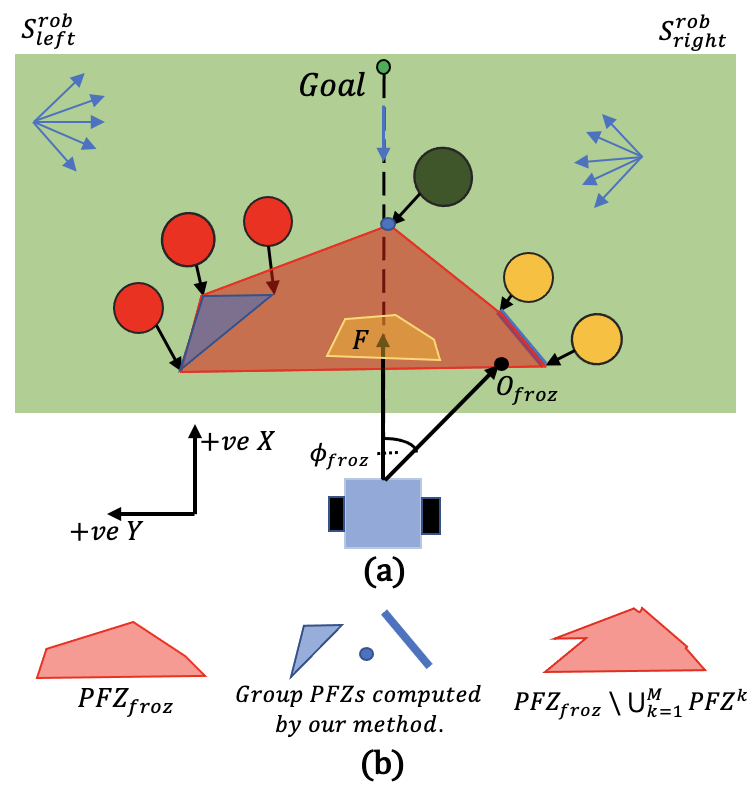}
      \caption {\small{\textbf{(a)} This scenario shows two groups (red and yellow agents in close proximity) and an individual pedestrian (green) walking towards the robot (blue box). The green rectangle denotes the robot's sensing region and the blue arrows denote the potentially freezing walking directions within each half of the sensing region ($S^{rob}_{left/right}$). \textbf{(b)} \textbf{[Left]} $PFZ_{froz}$ (a large 5-sided convex polygon in red) computed by Frozone \cite{frozone} while considering each individual as a separate obstacle. \textbf{[Middle]} Our CoMet-based approach identifies each group and computes the group PFZs ($PFZ^{k}$). This corresponds to the blue triangle for the agents in the red group, a line for the agents in the yellow group, a point for the green individual. \textbf{[Right]}  The region that represents $PFZ_{froz} \backslash \bigcup\limits_{k=1}^{M} PFZ^k$. These shapes are shown separately to observe the differences. Our proposed method is less conservative and results in a smaller deviation from PFZs (no deviation needed in this case) than Frozone~\cite{frozone}. We also highlight one possible subset of $PFZ_{froz} \backslash \bigcup\limits_{k=1}^{M} PFZ^k$, which contains positions with deviations $\le \phi_{froz}$}.}
      \label{fig:proof}
      \vspace{-15pt}
\end{figure}

\subsection{Socially-Compliant Navigation} \label{sec:comet-objective}
Our objective is to improve the naturalness of a robot's trajectory. We attribute three qualities to natural trajectories: (1) Not suddenly halting or freezing (avoiding FRP), (2) Low deviation angles, (3) Not cutting between high cohesion groups (friends, families etc) in a crowd. This is in accordance with humans' walking behaviors where people do not suddenly halt or significantly deviate from their goals \cite{ORCA}, and do not cut through high cohesion groups while walking \cite{group-permeability}. 

We extend Frozone \cite{frozone} (Section \ref{sec:frozone-background}) by considering groups and their cohesions, and prove that our proposed method leads to smaller deviations from the robot's goal, and shorter trajectory lengths. It also does not navigate the robot through high cohesion groups. We assume a higher density in the environment (in terms of crowds and static obstacles) than Frozone's formulation and human cooperation for the robot's navigation. Frozone prevents the robot from moving in front of a pedestrian to avoid slow down in terms of their walking speeds. 



\subsection{CoMet-based Navigation} \label{sec:comet-nav-details}
To improve the naturalness of a robot's trajectory, our proposed method includes the following steps: (1) Identifying potentially freezing groups within the sensing region of the robot and predicting their positions after a time period $t_{h}$; (2) Constructing a PFZ for each group using the predicted future locations of each group member (see blue regions in Fig. \ref{fig:proof});  (3) Computing a deviation angle to avoid the group PFZs while accounting for every group's cohesion. If a feasible solution is not found, the robot navigates between the group with the lowest cohesion in the scene. 

\vspace{5pt}

\textbf{Definition \ref{sec:comet-navigation}.1} (\textit{Potentially Freezing Groups:}) Groups of pedestrians that have a high probability of causing FRP after time $t_h$. Such groups are identified based on conditions of their average walking direction \cite{frozone} (see blue arrows in Fig. \ref{fig:proof}). Groups that satisfy these conditions move closer to the robot as time progresses (proven in \cite{frozone}). We predict the future positions of the potentially freezing group members as

\vspace{-15pt}
\begin{equation}
    \hat{\textbf{p}}^{i, t + t_{h}}_{pf} = \textbf{p}^{i, t}_{pf} + \textbf{v}^{G^{k, t}}_{avg}t_{h}, \quad i \in G^{k, t}, \quad k \in \{1, 2, ..., M.\}.
\end{equation}

\no Here, $\textbf{v}^{G^{k, t}}_{avg}$ is the average group walking vector, computed as the mean of the walking vectors of the group members and, M is the total number of potentially freezing groups. 

\vspace{5pt}
\textbf{Definition \ref{sec:comet-navigation}.2} (\textit{Group PFZ}) The region in the vicinity of a group where the robot has a high probability of freezing. Instead of constructing the single PFZ as the convex hull of all potentally freezing pedestrians (like in \cite{frozone}), we construct a PFZ for each potentially freezing group (see Fig. \ref{fig:proof}) as 

\begin{equation}
    PFZ^{k} = CH (\hat{\textbf{p}}^{i, t + t_{h}}_{pf}), \,\, i \in G^{k, t}, \,\, k \in \{1, 2, ..., M.\}.
\end{equation}

\no Every potentially freezing group's PFZ is computed for a future time instant $t + t_{h}$.

\subsubsection{Computing Deviation Angle}
If the robot's current trajectory navigates it into any of the group PFZs (implying an occurrence of FRP after time $t_h$), a deviation angle $\phi_{com}$ to avoid it is computed. The robot's current velocity $\textbf{v}_{rob}$ is deviated by $\phi_{com}$ using a rotation matrix about the Z-axis as, 

\vspace{-15pt}
\begin{gather}
    \textbf{v'}_{rob} = R_{z,\phi_{com}} \cdot \textbf{v}_{rob}, \\
    \phi_{com} = \underset{\textbf{v'}_{rob} \cdot t_{h} \notin PFZ^{k}}{\operatorname{argmin}} \left(dist(\textbf{v'}_{rob} \cdot t_{h}, \,\, \textbf{g}_{rob})\right) \label{eqn:deviation-1}. 
\end{gather}
This equation implies that our navigation method deviates the robot by the least amount from its goal such that it does not enter any group's PFZ. However, in dense scenarios, when the robot encounters many potentially freezing groups and their corresponding PFZs, Equation \ref{eqn:deviation-1} may not be able to compute a feasible solution for $\phi_{com}$. In such cases, a potential solution is to let the robot pass through the PFZ of a low cohesion group (see Fig. \ref{fig:real-scenarios}a). In such cases, we formulate the deviation angle as, 

\vspace{-15pt}
\begin{gather}
    \phi_{com} = \underset{\textbf{v'}_{rob} \cdot t_{h} \in \mathcal{P}}{\operatorname{argmin}} \left(dist(\textbf{v'}_{rob} \cdot t_{h}, \,\, \textbf{g}_{rob})\right), \\
    \mathcal{P} = PFZ^{min} \, \backslash \, (PFZ^{min} \cap PFZ^{k}) \,\, \forall k \in \{1,2,...,M.\},
\label{eqn:deviation-2}
\end{gather}
\vspace{-10pt}

\no where $PFZ^{min}$ is the PFZ of the group with the minimum cohesion in the scene. Since the permeability of low cohesion groups is high, the above formulation also lowers the probability of freezing.  


\begin{prop}
The deviation angles computed by CoMet-based navigation (Equations \ref{eqn:deviation-1} or \ref{eqn:deviation-2}), and Frozone (\ref{eqn:frozone-deviation}) satisfy the relationship $\phi_{froz} \ge \phi_{com}$.
\label{proposition1}
\end{prop}
\begin{proof}
Consider the scenario shown in Fig. \ref{fig:proof}, with $PFZ_{froz}$ depicted as a 5-sided convex polygon. A vertical line segment connects the robot ($O_{rob}$) to its goal and the robot's deviation is measured relative to it. Frozone deviates $\textbf{v}_{rob}$ such that this point lies on the boundary of $PFZ_{froz}$ (Equation \ref{eqn:phi1}) or in the exterior of $PFZ_{froz}$ (Equation \ref{eqn:phi2}), depending on whichever leads to a lower deviation. Let $O_{froz} = [x_{froz}, y_{froz}]$ be the point on the boundary of $PFZ_{froz}$ to which Frozone deviates (by angle $\phi_{froz}$) the robot. Based on Equations \ref{eqn:deviation-1} and \ref{eqn:deviation-2}, our CoMet-based solution deviates the robot to a point in $PFZ_{froz} \backslash \bigcup\limits_{k=1}^{M} PFZ^k$. 

Since,  $\bigcup\limits_{k=1}^{M} PFZ^k \subseteq PFZ_{froz}$, there exists a set $\mathbf{F} \subseteq PFZ_{froz} \backslash \bigcup\limits_{k=1}^{M} PFZ^k$ such that all positions in $\mathbf{F}$ lead to a deviation angle $\phi \le \phi_{froz}$. For instance, in Fig. \ref{fig:proof}, $\mathbf{F}$ can be a set just within the boundary of $PFZ_{froz}$ with Y-coordinates greater than $y_{froz}$. Since Equations (\ref{eqn:deviation-1}) and (\ref{eqn:deviation-2}) optimize for minimum deviation from the goal, $\phi_{com} \in \mathbf{F}$. This implies that $\phi_{com} \le \phi_{froz}$. The equality holds when $\mathbf{F} = \oldemptyset$ or when the closest edge of $PFZ_{froz}$ to the robot corresponds to the PFZ of a group. Based on the triangle inequality, shorter deviations lead to shorter trajectory lengths.
\end{proof} 
This bound also guarantees that our new navigation algorithm generates trajectory that are more natural than Frozone~\cite{frozone}. We integrate our CoMet-based navigation method with a DRL-based navigation scheme to evaluate it in real-world scenarios\cite{JiaPan1} and DWA \cite{DWA} to evaluate it in simulated environments. Figure \ref{fig:block}) shows the components of our navigation algorithm used to compute trajectories that are more natural in real-world scenes.

\begin{figure}[t]
      \centering
      \includegraphics[width=\columnwidth,height=3.75cm]{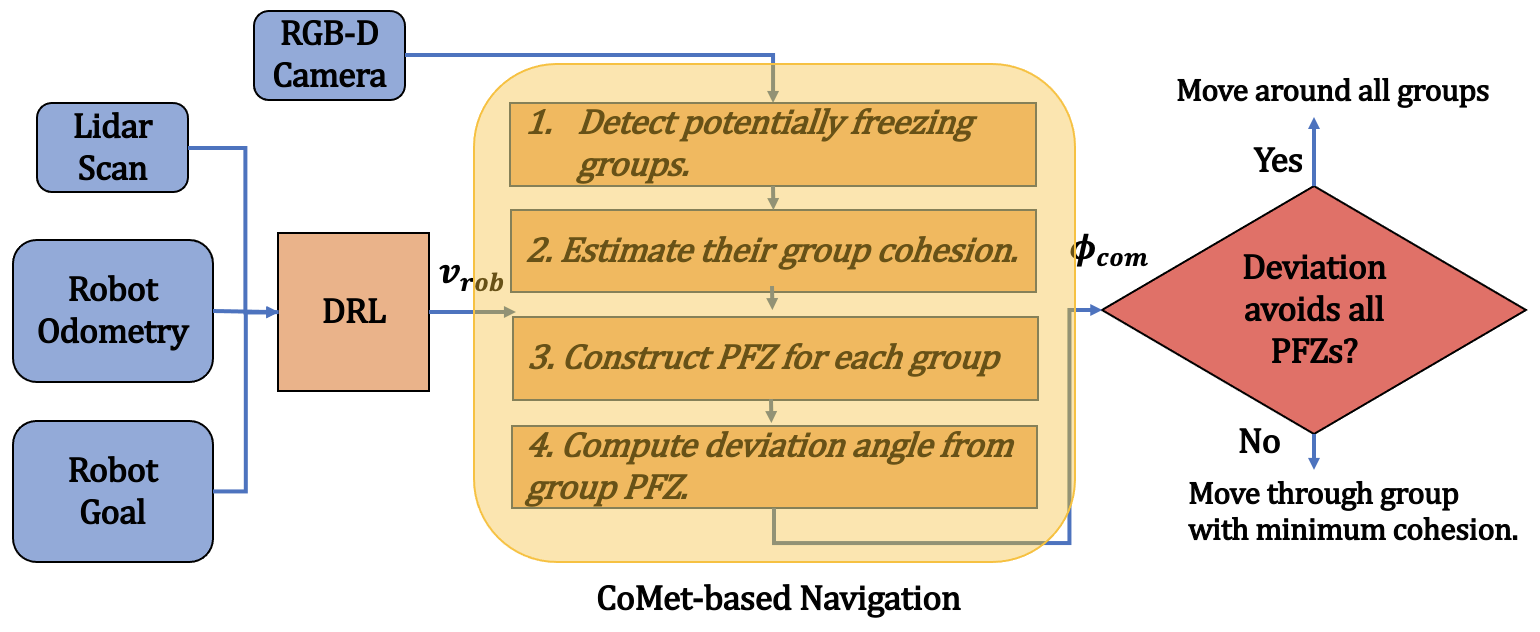}
      \caption {\small{Our Socially-Compliant Navigation Algorithm: We use a DRL framework used to guide the robot to its goal and handle static obstacles. CoMet-based navigation considers each group in the scene, identifies groups which could result in FRP, constructs group Potential Freezing Zones (PFZ), and computes a deviation angle to avoid such zones. Our formulation results in lower occurrence of freezing, lower deviations for the robot with respect to pedestrians and groups, and avoiding high cohesion groups by moving around them. In dense scenarios, when there is  no feasible solution for the deviation angle, our method navigates the robot through the group with the lowest cohesion. All these behaviors improve the naturalness of the robot trajectory's.}}
      \label{fig:block}
      \vspace{-15pt}
\end{figure}

\section{Implementation and Results}
In this section, we describe the implementation for computing group cohesion and socially-compliant navigation.  We then evaluate CoMet in different standard pedestrian datasets that are annotated with perceived group cohesion levels. We highlight the performance of our navigation algorithm and show the benefits over prior methods in terms of the following quantitative metrics:  freezing rate, deviation angle, and normalized path-length. We also qualitatively compare trajectories with an increased number of obstacles in the environment. 

\begin{figure*}[t]
\includegraphics[height=1.75in, width=\linewidth]{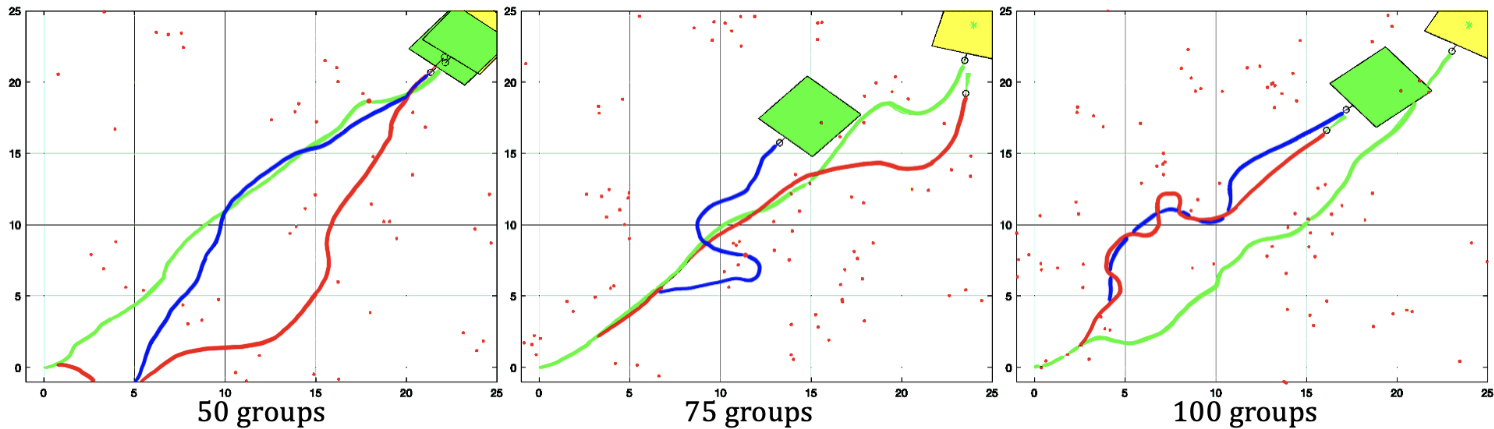}
\caption{\small{Qualitative evaluations of the trajectories generated using our algorithm (green), Frozone~\cite{frozone}(blue) and DWA\cite{DWA} (red) in scenarios with 50, 75 and 100 moving groups (red dots) that are non-cooperative for collision avoidance. The green and yellow squares represent the sensing regions of Frozone and CoMet-based navigation respectively. We observe that our formulation leads to lower deviation angles and more natural trajectories even in dense environments with tens to hundreds of obstacles. Both Frozone and DWA lead to unnatural trajectories with large deviations as the number of obstacles in the environment increases.}}
\label{fig:matlab-traj}
\vspace{-5pt}
\end{figure*}

\begin{figure*}[t]
\includegraphics[height=1.5in, width=\linewidth]{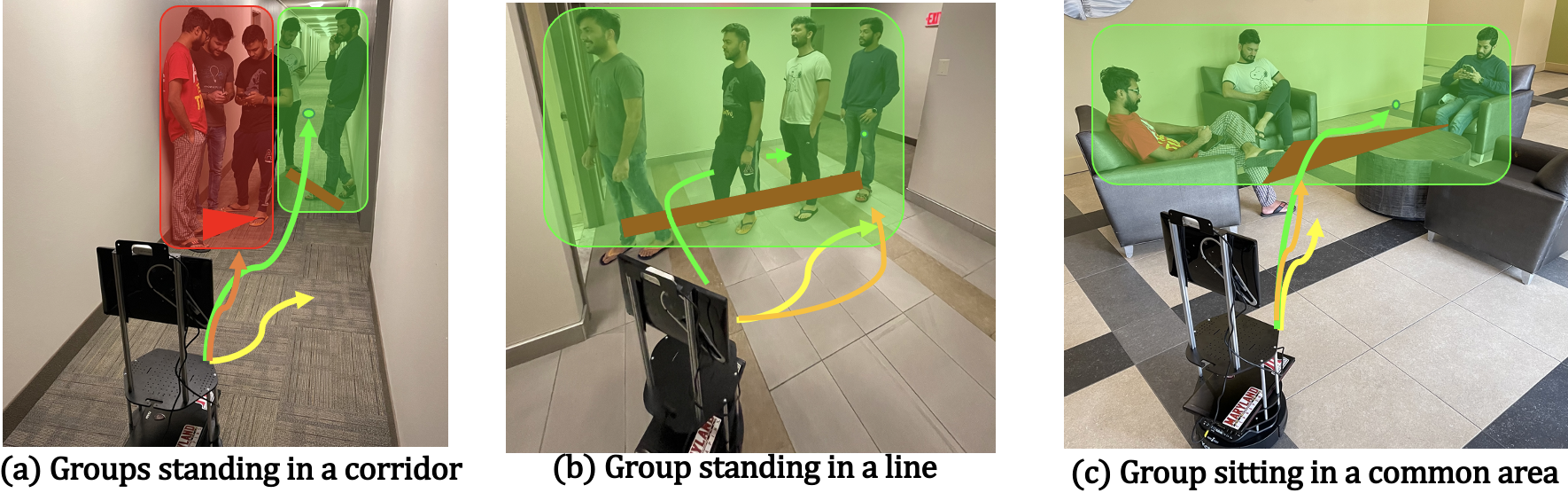}
\caption{\small{Qualitative evaluations of the trajectories generated using our algorithm (shown as green) in different scenarios. We also compare with the trajectories generated using Frozone~\cite{frozone}(shown as yellow) and a DRL-based algorithm ~\cite{JiaPan1}(shown as orange). Each group's PFZ is shown as a red region on the floor. We evaluate our method in three different real-world scenarios with tight spaces, with people standing or sitting. Our method differentiates between low (in green) and high cohesion (in red) groups, and navigates only between low cohesion groups. Frozone algorithm behaves in a conservative manner and halts the robot in dense scenarios. DRL~\cite{JiaPan1} prioritizes moving towards the goal and passes through high cohesion groups (see (a)). Overall, our approach results in socially-compliant trajectories.}}
\label{fig:real-scenarios}
\vspace{-5pt}
\end{figure*}

\subsection{Implementation}
In order to evaluate CoMet, we annotate groups in pedestrian datasets such as MOT, KITTI, ETH, etc. as low-, medium- and high-cohesion groups with the help of multiple human annotators. These annotated datasets are used as the ground truth since they reflect how humans perceive the groups in the videos. We choose these datasets because they depict groups in real-world scenarios with various lighting conditions, crowd densities and occlusions. We use a depth estimation method \cite{depth-from-rgb} with RGB images in the datasets to localize different pedestrians in the scene. We manually tune the weighing constants in the CoMet formulation ($K_p, K_w, K_s, K_i$) and set thresholds on the cohesion score to classify groups into the aforementioned categories based on the annotations in one of the datasets. We evaluate CoMet's precision and recall in the groups in all other datasets. Since there are no prior methods that compute cohesion, our evaluation is only against the human annotations.

We have evaluated our navigation algorithm using simulations created in MATLAB: (i) with tens to hundreds of groups, each with two to five members (Fig.\ref{fig:matlab-traj}) and a pre-assigned cohesion metric for each group and (ii) in corridor-like constrained scenarios with tens of pedestrians. The groups move linearly to a goal, and the simulated robot must take full responsibility to avoid collisions with them. We also evaluate our method on a real Turtlebot 2 robot mounted with an Intel RealSense RGB-D camera (for pedestrian tracking and localization) and a 2-D Hokuyo lidar (used by the DRL \cite{JiaPan1} method).

\subsection{Analysis}
\textbf{CoMet Classification:} Table \ref{Tab:Results2} highlights CoMet's classification precision and recall in multiple datasets. CoMet's parameters were tuned to improve its accuracy corresponding to detecting low-cohesion groups, which is required for navigation. This is reflected in the high precision and recall values in the second column.
CoMet observes pedestrians in these datasets for $\sim 5$ seconds. During this period, it is able to update its initial classification as pedestrians' trajectories change. For instance, a group initially perceived as high-cohesion may have its members move apart and is thereby classified as a low-cohesion group. Moreover, it is easier to detect features corresponding to proximity and walking speed than interactions between the pedestrians. This sometimes results in CoMet misclassifying medium- and high-cohesion groups in certain datasets. An interesting observation is that human annotators tend to classify groups in extremes, i.e. as either high-cohesion or low-cohesion groups. This leads to a low number of data points for medium-cohesion groups. This ground truth observation affects the effectiveness of our approach.

\begin{table}[t]
\resizebox{\columnwidth}{!}{
\begin{tabular}{|c|c|c|c|} 
\hline
\textbf{Dataset Video}\Tstrut \Tstrut & \textbf{Low-Cohesion} & \textbf{Medium-Cohesion} & \textbf{High-Cohesion} \\ [0.5ex] 
\hline
ADL-Rundle & 1.00/1.00 & - & 1.00/0.75 \\
\hline
AVG-TownCenter & 1.00/1.00 & 0.50/0.67 & 0.50/0.33 \\
\hline
ETH Jelmoli &  1.00/1.00 & - & 1.00/0.80\\
\hline
ETH Bahnhof & 1.00/1.00 & 1.00/0.67 & 0.80/1.00 \\
\hline
KITTI-16 & 1.00/1.00 & 1.00/0.50 & 0.33/1.00 \\
\hline
KITTI-17 & 1.00/1.00 & - & 1.00/0.667 \\
\hline
MOT17-08 & 1.00/1.00 & - & 1.00/0.50 \\
\hline
MOT17-11 & 1.00/1.00 & 0.75/1.00 & 1.00/0.833 \\
\hline
MOT20-02 & 1.00/1.00 & 0.34/0.33 & 0.50/0.33 \\
\hline
TUD & 1.00/0.857 & - & 1.00/1.00 \\
\hline
Venice & 0.875/0.875 & - & 0.938/0.60 \\
\hline
\end{tabular}}

\caption{ \small{\label{Tab:Results2} This table shows the precision (first value in each column) and recall (second) values for the three classes when CoMet observes groups in the video for $\sim 5$ seconds. CoMet's parameters have been tuned based on the ground truth in one of the datasets. We observe good accuracy for high-cohesion and low-cohesion groups. The original datasets have fewer occurrences of medium-cohesion groups, which impacts our precision.}}
\vspace{-10pt}
\end{table}

\textbf{Socially-Compliant Navigation:} We first qualitatively compare the trajectories (Fig.\ref{fig:matlab-traj}) of DWA \cite{DWA} (in red), a DWA-Frozone hybrid \cite{frozone}(in blue), and DWA and CoMet-based navigation (in green) in simulated environments with non-cooperative walking groups of obstacles. We observe that for the same set of dynamic obstacles, our approach computes smaller deviations while also avoiding collisions. DWA and the DWA-Frozone hybrids lead to conservative, and highly sub-optimal deviations from the goal direction. Although the exact set of obstacles each robot faces could be different depending on the trajectories they take, we observe that our method's deviations at any instant never exceed the deviations of the other two methods in comparison. These results reinforce proposition \ref{proposition1}.  

We also quantitatively compare the aforementioned three methods in environments with varying numbers ($10 - 50$) of pedestrians in a corridor-like scenario (which constricts the free space available to the robot). Pedestrians are given random initial locations and velocities, based on which they are classified into groups and $PFZ_{froz}$ and group PFZs ($PFZ^k$) are computed. The robot needs to navigate through the pedestrians to reach its goal. We use the following metrics: (1) \textit{Average deviation angle} measured relative to the line connecting the start and goal locations; (2) \textit{Freezing Rate} measured as the number of times the robot halted/froze over the total number of trials; and (3)  \textit{Normalized Path Length} measured as the robot's path length over the length of the line connecting the start and goal locations. 

Our method results in  lower values with respect to all these metrics, as compared to DWA and the DWA-Frozone hybrid on the same scenarios. As the crowd size, density, or number of groups increase, Frozone's conservative formulation makes the robot freeze at a high rate. This is because Frozone forbids the robot from avoiding a robot from in-front of the person, to improve pedestrian-friendliness in low- to medium-density scenes. On the other hand, we observe that our method produces trajectories with high social compliance and naturalness by reducing the occurrence of freezing behavior.

\begin{table}[t]
\resizebox{\columnwidth}{!}{
\begin{tabular}{|c|c|c|c|c|c|c|} 
\hline
\textbf{Metrics}\Tstrut \Tstrut & \textbf{Method} & \textbf{10 peds} & \textbf{20 peds} & \textbf{30 peds} & \textbf{40 peds} & \textbf{50 peds}  \\ [0.5ex] 
\hline
\multirow{3}{*}{\rotatebox[origin=c]{0}{\makecell{ \textbf{Avg. Deviation}\\\textbf{Angle}\\(lower better)}}} & DWA & 43.80 & 45.25 & 41.78 & 43.64 & 49.88 \\
 & DWA + Frozone & 46.41 & 46.67 & 42.85 & 50.80 & 53.22\\
 & DWA + Our Method & \textbf{41.07} & \textbf{44.31} & \textbf{41.11} & \textbf{37.47} & \textbf{34.19}\\
\hline
  
\multirow{3}{*}{\rotatebox[origin=c]{0}{\makecell{\textbf{Freezing Rate}\\(lower better)}}} & DWA & 0 & 0.37 & 0.39 & 0.36 & 0.44 \\
& DWA + Frozone & 0 & 0.29 & 0.25 & 0.43 & 0.57 \\
& DWA + Our Method & \textbf{0} & \textbf{0} & \textbf{0} & \textbf{0} & \textbf{0} \\
\hline

\multirow{3}{*}{\rotatebox[origin=c]{0}{\makecell{\textbf{Normalized}\\\textbf{Path Length}\\(lower better)}}}  & DWA & 1.33 & 1.31 & 1.42 & 1.47 & 1.49 \\
 & DWA + Frozone & 1.46 & 1.45 & 1.55 & 1.59 & 1.51\\
 & DWA + Our Method & \textbf{1.12} & \textbf{1.29} & \textbf{1.27} & \textbf{1.35} & \textbf{1.35}\\
\hline
\end{tabular}}

\caption{ \small{\label{Tab:Results3} We compare different navigation methods (DWA \cite{DWA}, DWA-Frozone hybrid ~\cite{frozone} and DWA-CoMet-based navigation algorithm based on three metrics. We observe that our method consistently results in lower values corresponding to all these metrics. This signifies improved naturalness of the robot's trajectory computed using our approach.}}
\vspace{-15pt}
\end{table}

\textbf{Real-world Evaluations:} We qualitatively compare our method's trajectories with Frozone's and a DRL algorithm's ~\cite{JiaPan1} trajectories . We highlight the differences in Fig. \ref{fig:cover-image}. Our approach is able to identify low-cohesion groups successfully and navigate through them without interfering with high-cohesion groups. In contrast, Frozone halts the robot completely, since it does not assume pedestrian cooperation in its formulation. The DRL method prioritizes reaching the goal with the minimum path length and therefore navigates through groups regardless of their cohesion. Therefore, the DRL algorithm can generate obtrusive trajectories.

\subsection{Other Improvements using CoMet}
Existing high-fidelity simulators for training Deep Reinforcement Learning methods for navigation simulate dynamic pedestrians as individual obstacles \cite{crowdsteer-ijcai}. Large-scale simulators \cite{menge-ros} use well-known motion models for individual pedestrians, and do not take into account model group behaviors. CoMet can be used to reverse engineer and simulate group behaviors based on different cohesion scores. For instance, simulated groups with high cohesion can have low intra-group member proximities, lower than average walking speeds, and larger group sizes. In 3-D simulators with human models, interactions can be modeled based on body orientations.

From an HRI standpoint, in social-distance monitoring robots \cite{covid-robot}, CoMet could help predict group properties such as inter-personal relationships between members based on their cohesions. This would help identify the groups that need to be issued warnings regarding maintaining social distancing. This leads to more apt interactions between the robot and the humans.

We present a novel method to compute the cohesion of a group 
of people in a crowd using visual features. We use our cohesion metric to design a novel robot navigation algorithm that results in socially-compliant trajectories. We highlight the benefits over previous algorithms in terms of the following metrics: reduced freezing, deviation angles, and path lengths.  We test our cohesion metric in annotated datasets and observe a high precision and recall.  

Our method has some limitations. We model cohesion through a linear relationship between the features, which may not work in all scenarios. In addition, there are other characteristics used to estimate cohesion, including age, gender, environmental context, cultural factors, etc. that we do not consider. Our approach also depends on the accuracy of how these features are detected, which may be affected due to lighting conditions and occlusions. Our navigation assumes that different groups exhibit varying levels of cohesions, which may not hold all the time. As part of future work, we hope to address these limitations and evaluate our approach in crowded real-world scenes.

\bibliographystyle{IEEEtran}
\bibliography{References}

\begin{thebibliography}{10}
\providecommand{\url}[1]{#1}
\csname url@rmstyle\endcsname
\providecommand{\newblock}{\relax}
\providecommand{\bibinfo}[2]{#2}
\providecommand\BIBentrySTDinterwordspacing{\spaceskip=0pt\relax}
\providecommand\BIBentryALTinterwordstretchfactor{4}
\providecommand\BIBentryALTinterwordspacing{\spaceskip=\fontdimen2\font plus
\BIBentryALTinterwordstretchfactor\fontdimen3\font minus
  \fontdimen4\font\relax}
\providecommand\BIBforeignlanguage[2]{{%
\expandafter\ifx\csname l@#1\endcsname\relax
\typeout{** WARNING: IEEEtran.bst: No hyphenation pattern has been}%
\typeout{** loaded for the language `#1'. Using the pattern for}%
\typeout{** the default language instead.}%
\else
\language=\csname l@#1\endcsname
\fi
#2}}

\bibitem{freezing1}
P.~{Trautman} and A.~{Krause}, ``Unfreezing the robot: Navigation in dense,
  interacting crowds,'' in \emph{2010 IEEE/RSJ International Conference on
  Intelligent Robots and Systems}, Oct 2010, pp. 797--803.

\bibitem{dyn-group-behavior}
L.~He, J.~Pan, S.~Narang, and D.~Manocha, ``Dynamic group behaviors for
  interactive crowd simulation,'' in \emph{Proceedings of the ACM
  SIGGRAPH/Eurographics Symposium on Computer Animation}.\hskip 1em plus 0.5em
  minus 0.4em\relax Eurographics Association, 2016, p. 139–147.

\bibitem{crowd-ped-dynamics-empirical}
S.~Bandini, A.~Gorrini, L.~Manenti, and G.~Vizzari, ``Crowd and pedestrian
  dynamics: Empirical investigation and simulation,'' in \emph{Proceedings of
  Measuring Behavior}, 08 2012.

\bibitem{socially-aware}
Y.~F. Chen, M.~Everett, M.~Liu, and J.~How, ``Socially aware motion planning
  with deep reinforcement learning,'' in \emph{2017 IEEE/RSJ International
  Conference on Intelligent Robots and Systems (IROS)}, 09 2017, pp.
  1343--1350.

\bibitem{WB1}
L.~{Tai}, J.~{Zhang}, M.~{Liu}, and W.~{Burgard}, ``Socially compliant
  navigation through raw depth inputs with generative adversarial imitation
  learning,'' in \emph{ICRA}, May 2018, pp. 1111--1117.

\bibitem{frozone}
A.~J. {Sathyamoorthy}, U.~{Patel}, T.~{Guan}, and D.~{Manocha}, ``Frozone:
  Freezing-free, pedestrian-friendly navigation in human crowds,'' \emph{IEEE
  Robotics and Automation Letters}, vol.~5, no.~3, pp. 4352--4359, 2020.

\bibitem{JHow1}
Y.~F. Chen, M.~Liu, M.~Everett, and J.~P. How, ``Decentralized
  non-communicating multiagent collision avoidance with deep reinforcement
  learning,'' in \emph{ICRA}.\hskip 1em plus 0.5em minus 0.4em\relax IEEE,
  2017, pp. 285--292.

\bibitem{densecavoid}
A.~{Jagan Sathyamoorthy}, J.~{Liang}, U.~{Patel}, T.~{Guan}, R.~{Chandra}, and
  D.~{Manocha}, ``{DenseCAvoid: Real-time Navigation in Dense Crowds using
  Anticipatory Behaviors},'' \emph{arXiv e-prints}, p. arXiv:2002.03038, Feb.
  2020.

\bibitem{friendsVsstrangers}
N.~Ashton, M.~E. Shaw, and A.~P. Worsham, ``Affective reactions to
  interpersonal distances by friends and strangers,'' \emph{Bulletin of the
  psychonomic society}, vol.~15, pp. 306--308, 1980.

\bibitem{self-loafing-group-cohesiveness}
S.~Karau and K.~Williams, ``The effects of group cohesiveness on social loafing
  and social compensation.'' \emph{Group Dynamics: Theory, Research, and
  Practice}, vol.~1, pp. 156--168, 1997.

\bibitem{group-permeability}
\BIBentryALTinterwordspacing
E.~Knowles, ``Boundaries around group interaction: the effect of group size and
  member status on boundary permeability,'' \emph{Journal of personality and
  social psychology}, vol.~26, no.~3, p. 327—331, June 1973. [Online].
  Available: \url{https://doi.org/10.1037/h0034464}
\BIBentrySTDinterwordspacing

\bibitem{hall-proxemics}
\BIBentryALTinterwordspacing
E.~T. Hall, R.~L. Birdwhistell, B.~Bock, P.~Bohannan, A.~R. Diebold, M.~Durbin,
  M.~S. Edmonson, J.~L. Fischer, D.~Hymes, S.~T. Kimball, W.~La~Barre, , J.~E.
  McClellan, D.~S. Marshall, G.~B. Milner, H.~B. Sarles, G.~L. Trager, and
  A.~P. Vayda, ``Proxemics [and comments and replies],'' \emph{Current
  Anthropology}, vol.~9, no. 2/3, pp. 83--108, 1968. [Online]. Available:
  \url{https://doi.org/10.1086/200975}
\BIBentrySTDinterwordspacing

\bibitem{friendly-dyads}
\BIBentryALTinterwordspacing
J.~Wagnild and C.~M. Wall-Scheffler, ``Energetic consequences of human
  sociality: Walking speed choices among friendly dyads,'' \emph{PLOS ONE},
  vol.~8, no.~10, pp. 1--6, 10 2013. [Online]. Available:
  \url{https://doi.org/10.1371/journal.pone.0076576}
\BIBentrySTDinterwordspacing

\bibitem{scene-ind-group-profiling}
J.~{Shao}, C.~C. {Loy}, and X.~{Wang}, ``Scene-independent group profiling in
  crowd,'' in \emph{2014 IEEE Conference on Computer Vision and Pattern
  Recognition}, 2014, pp. 2227--2234.

\bibitem{YOLOv3}
\BIBentryALTinterwordspacing
J.~Redmon and A.~Farhadi, ``Yolov3: An incremental improvement,'' \emph{CoRR},
  vol. abs/1804.02767, 2018. [Online]. Available:
  \url{http://arxiv.org/abs/1804.02767}
\BIBentrySTDinterwordspacing

\bibitem{densepeds}
R.~Chandra, U.~Bhattacharya, A.~Bera, and D.~Manocha, ``Densepeds: Pedestrian
  tracking in dense crowds using front-rvo and sparse features,'' \emph{arXiv
  preprint arXiv:1906.10313}, 2019.

\bibitem{disc-groups-imgs}
W.~Choi, Y.-W. Chao, C.~Pantofaru, and S.~Savarese, ``Discovering groups of
  people in images,'' in \emph{Computer Vision -- ECCV 2014}, D.~Fleet,
  T.~Pajdla, B.~Schiele, and T.~Tuytelaars, Eds.\hskip 1em plus 0.5em minus
  0.4em\relax Cham: Springer International Publishing, 2014, pp. 417--433.

\bibitem{social-groups-videos}
A.~K. {Chandran}, L.~A. {Poh}, and P.~{Vadakkepat}, ``Identifying social groups
  in pedestrian crowd videos,'' in \emph{2015 Eighth International Conference
  on Advances in Pattern Recognition (ICAPR)}, 2015, pp. 1--6.

\bibitem{crowd-features-video-seq}
R.~M. {Favaretto}, L.~L. {Dihl}, and S.~R. {Musse}, ``Detecting crowd features
  in video sequences,'' in \emph{2016 29th SIBGRAPI Conference on Graphics,
  Patterns and Images (SIBGRAPI)}, 2016, pp. 201--208.

\bibitem{comparison-behavior-features}
\BIBentryALTinterwordspacing
Z.~Ebrahimpour, W.~Wan, O.~Cervantes, T.~Luo, and H.~Ullah, ``Comparison of
  main approaches for extracting behavior features from crowd flow analysis,''
  \emph{ISPRS International Journal of Geo-Information}, vol.~8, no.~10, 2019.
  [Online]. Available: \url{https://www.mdpi.com/2220-9964/8/10/440}
\BIBentrySTDinterwordspacing

\bibitem{cluster-crowd-behavior}
M.~{Yang}, L.~{Rashidi}, A.~S. {Rao}, S.~{Rajasegarar}, M.~{Ganji},
  M.~{Palaniswami}, and C.~{Leckie}, ``Cluster-based crowd movement behavior
  detection,'' in \emph{2018 Digital Image Computing: Techniques and
  Applications (DICTA)}, 2018, pp. 1--8.

\bibitem{interaction-murino}
M.~Cristani, L.~Bazzani, G.~Paggetti, A.~Fossati, D.~Tosato, A.~Del~Bue,
  G.~Menegaz, and V.~Murino, ``Social interaction discovery by statistical
  analysis of f-formations,'' in \emph{Proceedings of the British Machine
  Vision Conference}.\hskip 1em plus 0.5em minus 0.4em\relax BMVA Press, 2011,
  pp. 23.1--23.12, http://dx.doi.org/10.5244/C.25.23.

\bibitem{Humanaware-survey}
T.~Kruse, A.~K. Pandey, R.~Alami, and A.~Kirsch, ``Human-aware robot
  navigation: A survey,'' \emph{Robotics Auton. Syst.}, vol.~61, pp.
  1726--1743, 2013.

\bibitem{Alahi}
C.~Chen, Y.~Liu, S.~Kreiss, and A.~Alahi, ``Crowd-robot interaction:
  Crowd-aware robot navigation with attention-based deep reinforcement
  learning,'' in \emph{ICRA}.\hskip 1em plus 0.5em minus 0.4em\relax IEEE,
  2019, pp. 6015--6022.

\bibitem{sociosense}
A.~{Bera}, T.~{Randhavane}, R.~{Prinja}, and D.~{Manocha}, ``Sociosense: Robot
  navigation amongst pedestrians with social and psychological constraints,''
  in \emph{2017 IEEE/RSJ International Conference on Intelligent Robots and
  Systems (IROS)}, Sep. 2017, pp. 7018--7025.

\bibitem{kim2015brvo}
S.~Kim, S.~J. Guy, W.~Liu, D.~Wilkie, R.~W. Lau, M.~C. Lin, and D.~Manocha,
  ``Brvo: Predicting pedestrian trajectories using velocity-space reasoning,''
  \emph{The International Journal of Robotics Research}, vol.~34, no.~2, pp.
  201--217, 2015.

\bibitem{robot-nav-crowd-drl}
L.~{Liu}, D.~{Dugas}, G.~{Cesari}, R.~{Siegwart}, and R.~{Dubé}, ``Robot
  navigation in crowded environments using deep reinforcement learning,'' in
  \emph{2020 IEEE/RSJ International Conference on Intelligent Robots and
  Systems (IROS)}, 2020, pp. 5671--5677.

\bibitem{JHow2}
M.~Everett, Y.~F. Chen, and J.~P. How, ``Motion planning among dynamic,
  decision-making agents with deep reinforcement learning,'' in
  \emph{IROS}.\hskip 1em plus 0.5em minus 0.4em\relax IEEE, 2018, pp.
  3052--3059.

\bibitem{JiaPan1}
P.~{Long}, T.~{Fan}, X.~{Liao}, W.~{Liu}, H.~{Zhang}, and J.~{Pan}, ``{Towards
  Optimally Decentralized Multi-Robot Collision Avoidance via Deep
  Reinforcement Learning},'' \emph{arXiv e-prints}, p. arXiv:1709.10082, Sep
  2017.

\bibitem{kitani}
K.~M. Kitani, B.~D. Ziebart, J.~A. Bagnell, and M.~Hebert, ``Activity
  forecasting,'' in \emph{Computer Vision -- ECCV 2012}, A.~Fitzgibbon,
  S.~Lazebnik, P.~Perona, Y.~Sato, and C.~Schmid, Eds.\hskip 1em plus 0.5em
  minus 0.4em\relax Berlin, Heidelberg: Springer Berlin Heidelberg, 2012, pp.
  201--214.

\bibitem{pfeiffer}
M.~Pfeiffer, U.~Schwesinger, H.~Sommer, E.~Galceran, and R.~Siegwart,
  ``Predicting actions to act predictably: Cooperative partial motion planning
  with maximum entropy models,'' 10 2016, pp. 2096--2101.

\bibitem{socially-compliant-1}
\BIBentryALTinterwordspacing
H.~Kretzschmar, M.~Spies, C.~Sprunk, and W.~Burgard, ``Socially compliant
  mobile robot navigation via inverse reinforcement learning,'' \emph{The
  International Journal of Robotics Research}, vol.~35, no.~11, pp. 1289--1307,
  2016. [Online]. Available: \url{https://doi.org/10.1177/0278364915619772}
\BIBentrySTDinterwordspacing

\bibitem{ped-dominance}
T.~Randhavane, A.~Bera, E.~Kubin, A.~Wang, K.~Gray, and D.~Manocha,
  ``Pedestrian dominance modeling for socially-aware robot navigation,'' in
  \emph{2019 International Conference on Robotics and Automation (ICRA)}, 2018.

\bibitem{entitivity}
A.~Bera, T.~Randhavane, A.~Wang, D.~Manocha, E.~Kubin, and K.~Gray,
  ``Classifying group emotions for socially-aware autonomous vehicle
  navigation,'' in \emph{The IEEE Conference on Computer Vision and Pattern
  Recognition (CVPR) Workshops}, June 2018.

\bibitem{curtis2016menge}
S.~Curtis, A.~Best, and D.~Manocha, ``Menge: A modular framework for simulating
  crowd movement,'' \emph{Collective Dynamics}, vol.~1, pp. 1--40, 2016.

\bibitem{govindaraju2005quick}
N.~K. Govindaraju, M.~C. Lin, and D.~Manocha, ``Quick-cullide: Fast inter-and
  intra-object collision culling using graphics hardware,'' in \emph{IEEE
  Proceedings. VR 2005. Virtual Reality, 2005.}\hskip 1em plus 0.5em minus
  0.4em\relax IEEE, 2005, pp. 59--66.

\bibitem{interaction-1961}
E.~Goffman, \emph{Encounters: Two studies in the sociology of
  interaction}.\hskip 1em plus 0.5em minus 0.4em\relax Ravenio Books, 1961.

\bibitem{yolo}
J.~{Redmon}, S.~{Divvala}, R.~{Girshick}, and A.~{Farhadi}, ``{You Only Look
  Once: Unified, Real-Time Object Detection},'' \emph{arXiv e-prints}, p.
  arXiv:1506.02640, June 2015.

\bibitem{deepsort}
N.~{Wojke}, A.~{Bewley}, and D.~{Paulus}, ``{Simple Online and Realtime
  Tracking with a Deep Association Metric},'' \emph{arXiv e-prints}, p.
  arXiv:1703.07402, Mar. 2017.

\bibitem{kalman-filter}
R.~E. Kalman, ``A new approach to linear filtering and prediction problems,''
  \emph{Transactions of the ASME--Journal of Basic Engineering}, vol.~82, no.
  Series D, pp. 35--45, 1960.

\bibitem{openface}
T.~{Baltrusaitis}, A.~{Zadeh}, Y.~C. {Lim}, and L.~{Morency}, ``Openface 2.0:
  Facial behavior analysis toolkit,'' in \emph{2018 13th IEEE International
  Conference on Automatic Face Gesture Recognition (FG 2018)}, 2018, pp.
  59--66.

\bibitem{ORCA}
J.~{Van Den Berg}, S.~Guy, M.~Lin, and D.~Manocha,
  ``\BIBforeignlanguage{English (US)}{Reciprocal n-body collision avoidance},''
  in \emph{\BIBforeignlanguage{English (US)}{Robotics Research - The 14th
  International Symposium ISRR}}, ser. Springer Tracts in Advanced Robotics,
  no. STAR, June 2011, pp. 3--19, 14th International Symposium of Robotic
  Research, ISRR 2009 ; Conference date: 31-08-2009 Through 03-09-2009.

\bibitem{DWA}
D.~{Fox}, W.~{Burgard}, and S.~{Thrun}, ``The dynamic window approach to
  collision avoidance,'' \emph{IEEE Robotics Automation Magazine}, vol.~4,
  no.~1, pp. 23--33, March 1997.

\bibitem{depth-from-rgb}
K.~Hars{\'a}nyi, A.~Kiss, A.~Majdik, and T.~Sziranyi, ``A hybrid cnn approach
  for single image depth estimation: A case study,'' in \emph{Multimedia and
  Network Information Systems}, K.~Choro{\'{s}}, M.~Kopel, E.~Kukla, and
  A.~Siemi{\'{n}}ski, Eds.\hskip 1em plus 0.5em minus 0.4em\relax Cham:
  Springer International Publishing, 2019, pp. 372--381.

\bibitem{crowdsteer-ijcai}
\BIBentryALTinterwordspacing
J.~Liang, U.~Patel, A.~J. Sathyamoorthy, and D.~Manocha, ``Crowd-steer:
  Realtime smooth and collision-free robot navigation in densely crowded
  scenarios trained using high-fidelity simulation,'' in \emph{IJCAI}, 7 2020,
  pp. 4221--4228, main track. [Online]. Available:
  \url{https://doi.org/10.24963/ijcai.2020/583}
\BIBentrySTDinterwordspacing

\bibitem{menge-ros}
A.~Aroor, S.~L. Epstein, and R.~Korpan, ``Mengeros: A crowd simulation tool for
  autonomous robot navigation,'' in \emph{AAAI 2017 Fall Symposium on
  Artificial Intelligence for Human-Robot Interaction}, 2017.

\bibitem{covid-robot}
A.~{Jagan Sathyamoorthy}, U.~{Patel}, Y.~{Ajay Savle}, M.~{Paul}, and
  D.~{Manocha}, ``{COVID-Robot: Monitoring Social Distancing Constraints in
  Crowded Scenarios},'' \emph{arXiv e-prints}, p. arXiv:2008.06585, Aug. 2020.

\end{thebibliography}
\end{document}